\documentclass[runningheads]{llncs}
\usepackage{graphicx}
\usepackage{tikz}
\usepackage{comment} 
\usepackage{amsmath,amssymb} 
\usepackage{color}

\usepackage{bm}
\usepackage{subfig}
\usepackage{capt-of}
\usepackage{booktabs}
\usepackage{mathtools}
\usepackage{siunitx}
\usepackage[figurename=Figure]{caption}
\usepackage{scrextend}
\usepackage{xspace}
\usepackage[implicit=false,colorlinks]{hyperref}
\usepackage{multirow}
\usepackage{thmtools}
\usepackage{floatrow}
\usepackage{wrapfig}
\newfloatcommand{capbtabbox}{table}[][\FBwidth]
\DeclareMathOperator*{\argmax}{arg\,max}
\DeclareMathOperator*{\argmin}{arg\,min}

\newcommand{\fig}{{Fig.}\@\xspace}
\newcommand{\tab}{{Table}\@\xspace}
\newcommand{\eqn}{{Eq.}\@\xspace}

\newcolumntype{L}[1]{>{\raggedright\arraybackslash}p{#1}}
\newcolumntype{C}[1]{>{\centering\arraybackslash}p{#1}}
\newcolumntype{R}[1]{>{\raggedleft\arraybackslash}p{#1}}

\newcommand{\etal}{et al.\@\xspace}
\newcommand{\ie}{{i.e.,}\@\xspace}
\newcommand{\eg}{{e.g.,}\@\xspace}
\newcommand{\etc}{{etc.}\@\xspace}

\begin{document}
\pagestyle{headings}
\mainmatter
\def\ECCVSubNumber{615}  

\title{$n$-Reference Transfer Learning for Saliency Prediction} 

%
\author{Yan~Luo\inst{1}\orcidID{0000-0001-5135-0316} \and
Yongkang~Wong\inst{2}\orcidID{0000-0002-1239-4428} \and
Mohan~S.~Kankanhalli\inst{2}\orcidID{0000-0002-4846-2015} \and
Qi~Zhao\inst{1}\orcidID{0000-0003-3054-8934}}
\authorrunning{Y. Luo et al.}
%
\institute{Department of Computer Science and Engineering, University of Minnesota \\
\email{luoxx648@umn.edu}, \email{qzhao@cs.umn.edu}\\
\and
School of Computing, National University of Singapore \\
\email{\{wongyk, mohan\}@comp.nus.edu.sg}}
\maketitle

\begin{abstract}

Benefiting from deep learning research and large-scale datasets, saliency prediction has achieved significant success in the past decade. 
However, it still remains challenging to predict saliency maps on images in new domains that lack sufficient data for data-hungry models. 
To solve this problem, we propose a few-shot transfer learning paradigm for saliency prediction, which enables efficient transfer of knowledge learned from the existing large-scale saliency datasets to a target domain with limited labeled samples. 
Specifically, few target domain samples are used as the {\it reference} to train a model with a source domain dataset such that the training process can converge to a local minimum in favor of the target domain. 
Then, the learned model is further fine-tuned with the {\it reference}.
The proposed framework is gradient-based and model-agnostic.
We conduct comprehensive experiments and ablation study on various source domain and target domain pairs. 
The results show that the proposed framework achieves a significant performance improvement. 
The code is publicly available at
\url{https://github.com/luoyan407/n-reference}.

\keywords{Deep learning, Saliency prediction, n-shot transfer learning}

\end{abstract}

\section{Introduction}

Saliency prediction is the task that aims to model human attention to predict where people look in the given image. 
Thanks to the power of deep neural networks~\cite{He_CVPR_2016,Krizhevsky_NIPS_2012,Xie_CVPR_2017} (DNNs), state-of-the-art saliency models~\cite{Cornia_TIP_2018,Yang_arXiv_2019} perform very well in predicting human attention on naturalistic images. 
Behind the success of this task, a considerable amount of real-world images and corresponding human fixations fuels the process of training the data-hungry DNNs.

\begin{figure}
	\centering
	\includegraphics[width=1.0\linewidth]{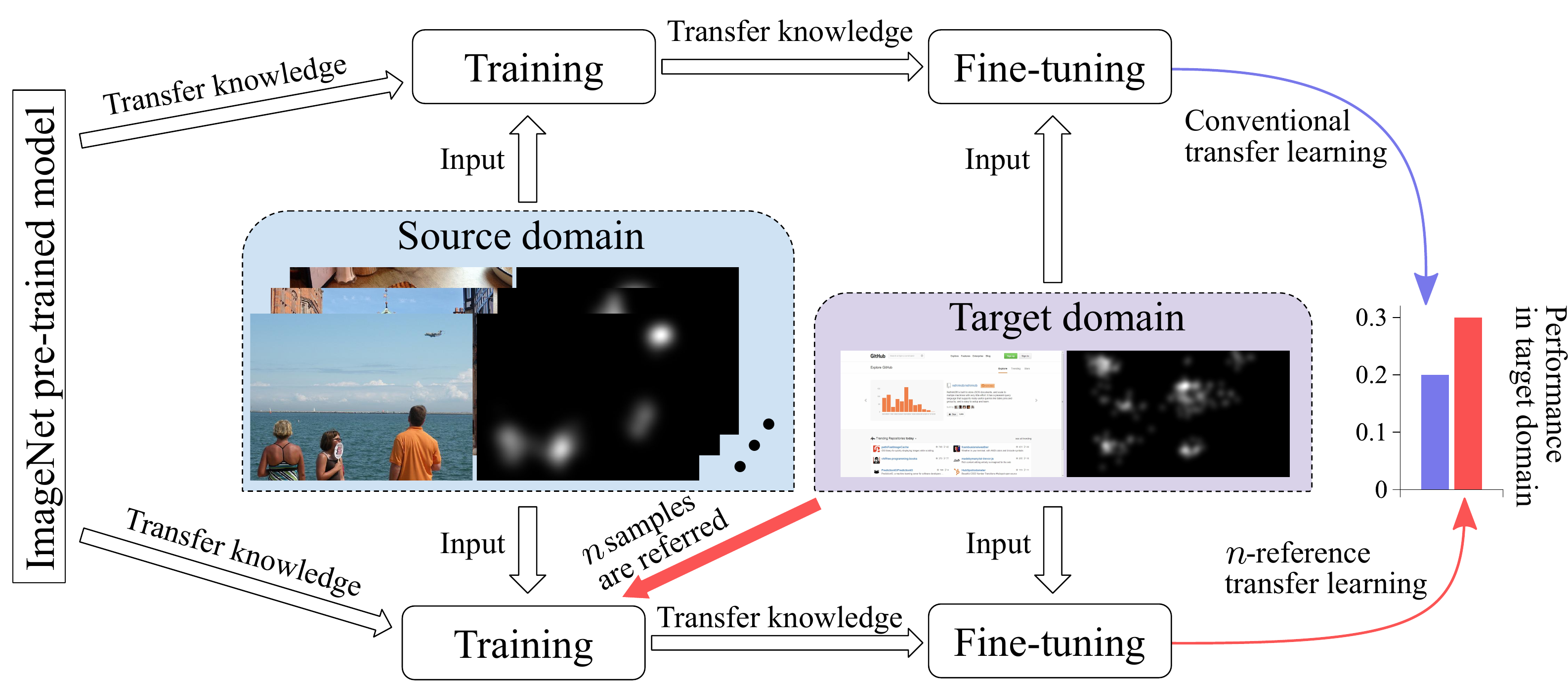}
	\caption{The proposed $n$-reference transfer learning framework for saliency prediction. 
    This framework aims to generate a better initialization with $n$ reference samples from the target domain when training on the source domain, followed by fine-tuning to maximize knowledge transfer. 
    It is based on the widely-used two-stage transfer learning framework (\ie~first training and then fine-tuning) and can easily adapt to other fine-tuning strategies}
	\label{fig:teaser}
\end{figure}

However, it is still difficult to predict saliency maps on images in novel domains, which has insufficient or few data to train saliency models with desired performance. 
As the time/money cost of collecting human fixations is prohibitive \cite{Borji_arXiv_2015,Jiang_CVPR_2015}, a feasible solution is to reuse the existing large-scale saliency datasets along with a few target domain samples to solve this problem. 
Along this line, we study how to transfer the knowledge learned from the existing large-scale saliency datasets to the target domain in a few-shot transfer learning setting.

The necessity of few-shot transfer learning for saliency prediction lies in the nature of the task. 
Based on findings drawn from the behavioral experiments, the way that humans attend to regions is significantly affected by the scene context~\cite{Neider_VR_2006,Torralba_PR_2006,Wolfe_NHB_2017}. 
The scene context is correlated to the image domain~\cite{Shrivastava_ToG_2011}. 
In other words, each image from a specific domain could be representative of the others from the same domain to some degree, \eg~webpage images generally have a similar layout and design~\cite{Shan_ICISBDE_2017}. 
In visual saliency study, existing datasets~\cite{Borji_arXiv_2015,Shen_ECCV_2014} in non-natural images domain are much smaller than the natural image ones~\cite{Jiang_CVPR_2015,Judd_ICCV_2009}. 
Moreover, there are numerous images used in the subfields of medicine, biology, \etc, which may not have any human fixation data yet. 
In this work, we assume that it is feasible and viable to collect human fixations on a small number of images to enable few-shot learning.

Compared to $n$-reference transfer learning for classification task~\cite{Bauml_ICRA_2019}, we focus on how to use very few target domain samples as references to learn a better initial model for fine-tuning. 
Moreover, there exists no such works for saliency prediction task. 
Models designed for classification may not work for saliency prediction.
First, visual samples in existing classification tasks often contain limited visual concepts (\ie~pre-defined object classes), while objects of any class may appear in the images used for saliency prediction. 
In this sense, saliency prediction often handles images with higher diversity than the ones used for classification.
Second, the output of classification models~\cite{Bauml_ICRA_2019,He_CVPR_2016,Krizhevsky_NIPS_2012,Li_CVPR_2019,Sung_CVPR_2018} is a discrete label, while saliency models \cite{Cornia_TIP_2018,Yang_arXiv_2019} output a matrix of real numbers.

In this work,
we follow the widely-used two-stage transfer learning framework \cite{Bauml_ICRA_2019,Guo_CVPR_2019,Shan_ICISBDE_2017}, \ie first training and then fine-tuning, and propose a $n$-reference transfer learning framework. 
Specifically, in the training stage, it aims to use a small number of samples in the target domain as references to guide the knowledge learned from the source domain dataset. 
In this way, the learned model is adapted to the target domain and can be seen as a better initialization than the one trained without the references.
The small number of target domain samples are used as references in both the training stage and as the training data in the fine-tuning stage.
The proposed framework is shown in \fig~\ref{fig:teaser}.

Mathematically, we use cosine similarity between two gradients to facilitate the reference aware model training, 
where the two gradients are respectively computed by samples in the source and target domain.
If the angle between the two gradients is greater than 90 degrees, 
which implies that the directions of the model update are significantly different from each other, 
we optimize the gradient for the update to have smaller differences with the target-domain referenced gradient in cosine similarity. 
The intuition behind is to mimic the process of human learning with the reference sample, \ie~we adaptively learn from new information so that the newly absorbed knowledge will not contradict the observation of the reference samples \cite{Luo_TPAMI_2019,Li_WACV_2020}. 
The proposed framework is gradient-based and it is model-agnostic.

To comprehensively evaluate the proposed framework, we employ SALICON \cite{Jiang_CVPR_2015} and MIT1003~\cite{Judd_ICCV_2009} as the source domain datasets (\ie~the knowledge sets), and WebSal~\cite{Shen_ECCV_2014} and the art subset in CAT2000~\cite{Borji_arXiv_2015} as the target domain data. 
We randomly select 1, 5, or 10 samples from the target domain data as references.
The contributions of this work can be summarized as follows:
\begin{itemize}
	\item To study how humans perceive scenes from a partially explored domain, we propose a model-agnostic few-shot transfer learning paradigm to transfer knowledge from the source domain to the target domain. 
	This is the first work that studies few-shot transfer learning for saliency prediction.
	\item We propose a $n$-reference transfer learning framework to adaptively guide the training process. It guarantees that the knowledge learned with the source domain data would not contradict the references in the target domain, and produce a good initialization for further fine-tuning. The proposed framework is model-agnostic and can generally work with existing saliency models.
	\item Comprehensive experiments show the proposed framework works on various combinations of source domain and target domain pairs. 
	The experiment with various baseline models show that the proposed approach can efficiently transfer the knowledge from the source domain to the target domain.
\end{itemize}
\section{Related Works}
\label{sec:related}

\subsection{Saliency Prediction}

Saliency prediction aims to mimic human vision system to perceive interesting regions in a cluttered visual world. 
Itti~\etal~\cite{Itti_PAMI_1998} develop the first bottom-up stimulus-driven saliency model. 
Since then, many works emerge to interpret visual saliency from various perspectives~\cite{Harel_NIPS_2007,Hou_PAMI_2011,Judd_ICCV_2009,Zhang_CVPR_2013}. 
With the advent of DNNs~\cite{He_CVPR_2016,Krizhevsky_NIPS_2012,Xie_CVPR_2017}, saliency prediction benefitted from data-driven discriminative features instead of relying on hand-crafted features \cite{Cornia_ICPR_2016,Kruthiventi_TIP_2017,Kummerer_ICLR_2015,Pan_arXiv_2017}. 
Recently, Cornia~\etal~\cite{Cornia_TIP_2018} introduce a network that integrates ResNet-50~\cite{He_CVPR_2016} and convolutional LSTMs to better attend to salient regions by iteratively refining the predictions. Yang~\etal~\cite{Yang_arXiv_2019} propose a dilated inception network (DINet) that stacks dilated convolutions with different dilation rates upon ResNet-50 to capture wider spatial information. 
It achieves state-of-the-art performance on various benchmarks. 
A widely-used practice to transfer the knowledge learned from image classification to saliency prediction is by using the weights pre-trained on ImageNet as model initialization~\cite{Cornia_ICPR_2016,Kruthiventi_TIP_2017,Kummerer_ICLR_2015,Pan_arXiv_2017}.
In contrast, this work studies the few-shot cross-domain transfer learning problem, which takes place between two domains.
Without loss of generality, we follow \cite{Luo_TPAMI_2019} to adopt both ResNet-50 and DINet as the baseline models in this work.

\subsection{Few-shot Learning}

Few-shot learning~\cite{Li_PAMI_2006,Lake_AMCSS_2011,Li_CVPR_2019,Sung_CVPR_2018} aims to study how to learn classifiers for unseen visual concepts with only a few samples per class. 
Lake~\etal~\cite{Lake_Science_2015} introduce a Bayesian program learning framework that can learn from one example for predicting character strokes. 
Matching networks~\cite{Vinyals_NIPS_2016} use an attention mechanism that is analogous to a kernel density estimator so that it can learn from a few examples rapidly. 
Sung~\etal~\cite{Sung_CVPR_2018} propose a relation network to learn a transferable deep metric to compare the relation between the small number samples. 
In \cite{Lee_CVPR_2019}, Lee~\etal study how to learn feature embeddings with a few samples that can minimize generalization error across a distribution of tasks. 
As the process of collecting human fixations is prohibitive~\cite{Jiang_CVPR_2015}, 
learning with very few samples is promising for saliency prediction to overcome the need for big data.

\subsection{Transfer Learning}

Transfer learning, a.k.a. domain adaptation or domain transfer, is a paradigm to utilize training data in the source domain to solve the problem in the target domain~\cite{Csurka_Springer_2017,Daume_JAIR_2006,Li_ACMMM_2017,Shan_ICISBDE_2017,Pan_TNN_2010}. 
In general, it can be seen as a two-stage learning framework, \ie first training a model with source domain data and then fine-tuning the pre-trained model with target domain data. 
There are many DNN-based works~\cite{Bauml_ICRA_2019,Bengio_ICMLW_2012,Ge_CVPR_2017,Guo_CVPR_2019,Li_ICML_2018} that use this learning framework for classification tasks. 
Specifically, Guo \etal \cite{Guo_CVPR_2019} study and design a variant of the standard fine-tuning method for better transferability. 
However, it requires many training samples to determine whether it should fine-tune or freeze the parameters in a particular layer. 
Recently, B{\"a}uml and Tulbure~\cite{Bauml_ICRA_2019} introduce a learning framework that transfers the knowledge learned from the source domain to the target domain with a few samples for tactile material classification. 
As saliency prediction is by nature class-agnostic, 
learning to predict human fixations with very few samples (\eg $\le 10$) in the target domain is more challenging than the same paradigm for classification and has not been explored yet.
Different from the aforementioned methods, we propose the first model-agnostic few-shot transfer learning framework for saliency prediction and conduct comprehensive study on multiple combinations of source domain datasets and target domain datasets.
\section{Methodology}
\label{sec:method}

In this section, we first formulate the problem and discuss its theoretical generalization bound. 
Then, we delve into the details of the proposed framework.

\subsection{Problem Statement}

In this work, we denote the images as $I^{S},I^{T} \in \mathbb{R}^{m}$ and the human fixation maps as $y^{S}, y^{T} \in \mathcal{Y}$ ($\mathcal{Y} \equiv [0,1]^{m}\subseteq \mathbb{R}^{m}$), where $m$ is the dimensions of the image and $S$ ($T$) indicates the source (target) domain. In general, given an image $I$, the prediction function $f: \mathbb{R}^{m} \xrightarrow{\theta} \mathcal{Y}$ with parameters $\theta$ will predict $z$ and then the loss function $\ell: \mathcal{Y} \times \mathcal{Y}\rightarrow \mathbb{R}_{+}$ will evaluate the discrepancy between $z$ and $y$. 
Transfer learning for saliency prediction task can be considered as a two-stage learning problem. First, the model's parameters are learned with the source domain data through the training process, \ie
\begin{align}
    \theta_\mathsf{TR} = \argmin_{\theta}\ \frac{1}{|D^{S}|}\sum_{(I_{i},y_{i})\in D^{S}}^{}\ell(f(I_{i};\theta),y_{i})|_{\theta_{0}}
\label{eqn:stage1}
\end{align}
where $D^{S}$ is the source domain dataset, $|D^{S}|$ is the number of the samples, and $\mathsf{TR}$ stands for training. $\theta_{0}$ are the initialized parameters and the model is usually pre-trained on ImageNet \cite{Deng_CVPR_2009}. Then, $\theta_\mathsf{TR}$ is taken as the initialization for further fine-tuning on the target domain data, \ie
\begin{align}
    \theta_\mathsf{FT}^{*} = \argmin_{\theta}\ \frac{1}{|D^{T}|}\sum_{(I_{i},y_{i})\in D^{T}}^{}\ell(f(I_{i};\theta),y_{i})|_{\theta_{0}=\theta_\mathsf{TR}}
\label{eqn:stage2}
\end{align}

In this work, we aim to learn a better initialization by the first stage objective (\ref{eqn:stage1}), which is in favor of the target domain data. Such initialized parameters (\ie $\theta_\mathsf{TR}$) are expected to further achieve better performance by fine-tuning on $D^{T}$. To this end, we introduce a referencing mechanism that allows the training process fed with $D^{S}$ to reference the model update w.r.t. the referenced samples $(I^{R},y^{R}) \in D^{T} (|D^{S}| \gg |D^{T}|)$. Mathematically, this can be formulated as
\begin{align}
\theta_\mathsf{TR-Ref} = \argmin_{\theta}\ \frac{1}{|D^{S}|}\sum_{\substack{(I_{i},y_{i})\in D^{S}\\(I_{j}^{R},y_{j}^{R})\in D^{T}}}^{}\ell(f_\mathsf{Ref}(I_{i};\theta,(I_{j}^{R},y_{j}^{R})),y_{i})|_{\theta_{0}}
\label{eqn:stage1_ref}
\end{align}
where $\mathsf{TR\!-\!Ref}$ indicates the training process references target domain samples when updating the model.
$f_\mathsf{Ref}$ is a variant of $f$ which has the same forward propagation as $f$ but has more complicated backward propagation. 
$\theta_\mathsf{TR-Ref}$ is taken as the initialization in the second stage objective (\ref{eqn:stage2}) for further fine-tuning. We denote the resulting parameters as $\theta_\mathsf{FT|Ref}$. 

\subsection{Generalization Bound of Saliency Prediction}

Here, we discuss the theoretical guarantee of saliency prediction.
Following the setting used in \cite{Mohri_MIT_2012}, given training data ${(I_{1}, y_{1}),(I_{2}, y_{2}),\ldots} \in \mathcal{X} \times \mathcal{Y}$, where $\mathcal{Y}\in [0,1]^{m}\subseteq \mathbb{R}^{m}$, we use the $L^p$ loss, \ie $\ell^{p}: \mathcal{Y} \times \mathcal{Y} \rightarrow \mathbb{R}_{+}, p\ge1$. The prediction function $f(\cdot;\theta)$ is denoted as $f(\cdot)$ for simplicity. $I$ is drawn i.i.d. according to the unknown distribution $\mathcal{D}$ and $y=f^{*}(I)$ where $f^*$ is the target labeling function. Saliency prediction can be considered as a mathematical problem that finds hypothesis $f: \mathbb{R}^{m}\rightarrow [0,1]^{m}$ in a set $H$ with small generalization error w.r.t. $f^{*}$,
\begin{align*}
R_{\mathcal{D}}(f)=E_{I\sim \mathcal{D}}[\ell(f(I), f^{*}(I))].
\end{align*}
In practice, as $\mathcal{D}$ is unknown, we use empirical error for approximation, \ie
\begin{align*}
\hat{R}_{\mathcal{D}}(f)=\frac{1}{|D|}\sum_{i=1}^{|D|}\ell(f(I_{i}), y_{i}),
\end{align*}
where $|D|$ is the sample number in dataset $D$ for training.

We introduce the generalization bound of saliency prediction as follows.
The proof is provided in the supplementary document.
\begin{theorem}[Saliency generalization bound]
	Denote $H$ as a finite hypothesis set. Given $\ell^{p}$ and $y\in [0,1]^{m}$, for any $\delta>0$, with probability at least $1-\delta$, the following inequality holds for all $f\in H$:
	\begin{align*}
	|R_{\mathcal{D}}(f) - \hat{R}_{\mathcal{D}}(f)| \le  m^{\frac{1}{p}}\sqrt{\frac{\log|H|+\log\frac{2}{\delta}}{2|D|}}
	\end{align*}
	\label{thrm:gb}
\end{theorem}
\begin{remark}
	Theorem~\ref{thrm:gb} shows how the training set scale influences the generalization bound. When $|D|$ tends towards infinity, $R_{\mathcal{D}}(f) \equiv \hat{R}_{\mathcal{D}}(f)$. This conforms to the general intuition that it can train a more general model with more data. Contrarily, when $|D|=1$, it leads to the largest bound for $|R_{\mathcal{D}}(f) - \hat{R}_{\mathcal{D}}(f)|$. Moreover, it demonstrates the task is challenging with small number of samples.
\end{remark}

\begin{figure}[!t]
	\centering
	\includegraphics[width=0.96\textwidth]{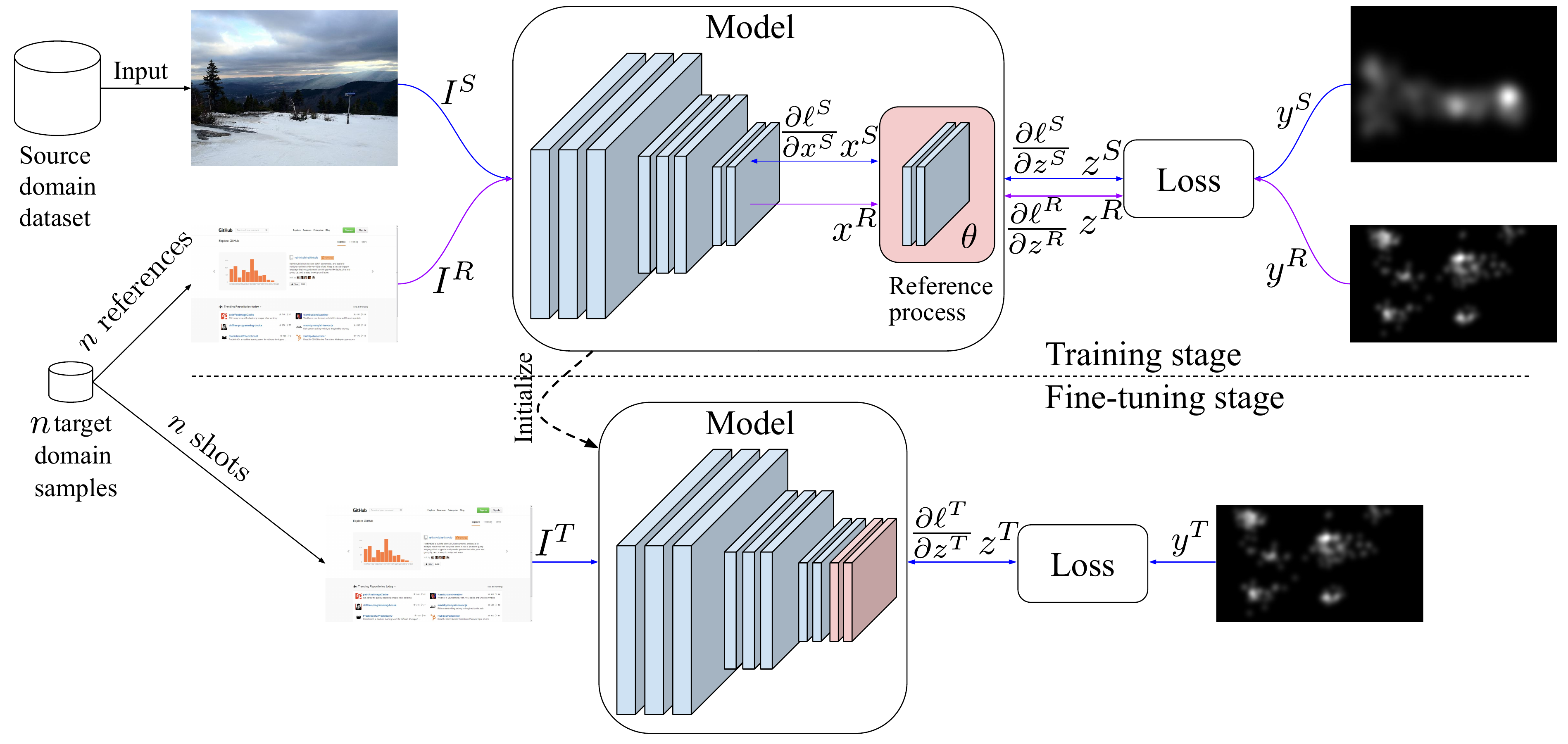}
	\caption{Proposed $n$-reference transfer learning framework. Note that we assume that only very few samples from the target domain are available, \ie~$n\le 10$}
	\label{fig:framework}
\end{figure}

\subsection{Overall Framework}

In this subsection, we introduce the few-shot transfer learning framework that solves the objective function (\ref{eqn:stage2}) and (\ref{eqn:stage1_ref}).
The overall workflow of the proposed $n$-reference transfer learning framework is shown in \fig~\ref{fig:framework}.

Similar to classification model~\cite{He_CVPR_2016,Huang_CVPR_2017,Xie_CVPR_2017}, state-of-the-art saliency models tend to be large. For example, DINet~\cite{Yang_arXiv_2019} and SAM-ResNet-50~\cite{Cornia_TIP_2018} consist of 26M and 70M parameters, respectively. Therefore, instead of inefficiently applying the proposed framework to the whole saliency model, we only apply it to a few downstream layers which are close to the output. The downstream layers produce discriminative features used for prediction with a small number of parameters, and it makes the transfer learning process more cost-effective. Consequently, we split the model into two parts, \ie the model body $\theta_\mathsf{body}$ and the model head $\theta_\mathsf{head}$. This split would be only effective in the training stage and the two parts will be integrated again as they always are in the inference stage. Note that the split is flexible. The effective scope of the proposed framework could cover the whole model and the model body would correspondingly turn to be an empty set. As we only focus on $\theta_\mathsf{head}$, we simplify it as $\theta$ in the following text. 

In the forward propagation, as the training image $I^{S} \in D^{S}$ and the reference image $I^{R} \in D_{T}$ are fed to the model body, the discriminative feature $x^{S}$ and $x^{R}$ are generated, respectively. Then, the model head would take $x^{S}$ and $x^{R}$ as input to produce prediction $z^{S}$ and $z^{R}$, respectively. Specifically, $z^{S} = f(x^{S};\theta)$. A similar process applies to $z^{R}$. The loss function is used to compute the distance between $z^{S}$ and $y^{S}$ (and between $y^{R}$ and $y^{R}$ as well). In the backward propagation, two gradients are computed by the chain rule
	\begin{align*}
	\frac{\partial \ell^{S}}{\partial \theta} = \frac{\partial \ell(f(x^{S};\theta), y^{S})}{\partial z^{S}} \frac{\partial z^{S}}{\partial \theta}, \hspace{4ex} \frac{\partial \ell^{R}}{\partial \theta} = \frac{\partial \ell(f(x^{R};\theta), y^{R})}{\partial z^{R}} \frac{\partial z^{R}}{\partial \theta}.
	\end{align*}
Specifically, $\frac{\partial \ell^{S}}{\partial \theta}$ indicates the model update towards a local minimum $\theta^{*(S)}$ which is learned from the samples from $D^{S}$, while $\frac{\partial \ell^{R}}{\partial \theta}$ indicates the model update towards a local minimum $\theta^{*(T)}$ which is learned from the samples from $D^{T}$. 

As shown in \fig~\ref{fig:framework}, $\theta_\mathsf{head}$ are updated by the proposed reference process and $\theta_\mathsf{body}$ are updated with the standard gradients in the training stage. 
During fine-tuning, $\theta_\mathsf{head}$ and $\theta_\mathsf{body}$ are updated with the standard gradients.

\begin{figure}[!t]
	\centering
	\includegraphics[width=0.85\linewidth]{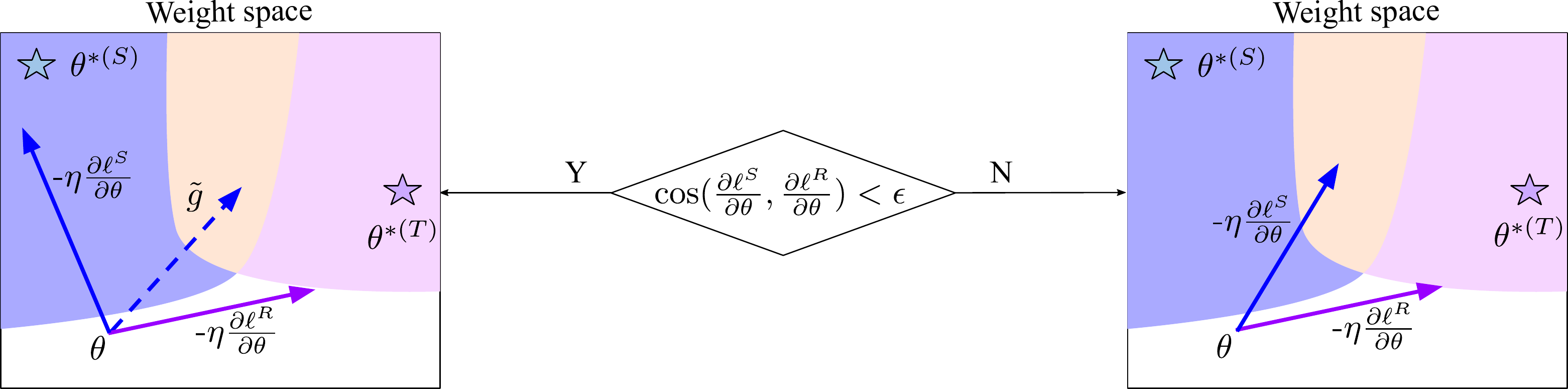}
	\caption{The reference process computing the gradient that better adapts to the target domain data. $\theta^{*(S)}$ is a local minimum trained by sufficient source domain samples, while $\theta^{*(T)}$ is a local minimum trained by sufficient target domain samples. Given a pre-defined threshold $\epsilon$, if the cosine similarity between the gradient ($\frac{\partial \ell^{S}}{\partial \theta}$) generated by the source sample and the gradient ($\frac{\partial \ell^{R}}{\partial \theta}$) generated by the reference sample is smaller than $\epsilon$, it will compute a corrected gradient by optimizing the cosine similarity. It retains $\frac{\partial \ell^{S}}{\partial \theta}$ otherwise}
	\label{fig:reference}
\end{figure}

\subsection{Reference Process}

Here, we delve into the formulation of the proposed reference process (\fig~\ref{fig:reference}).
The cosine similarity between $\frac{\partial \ell^{S}}{\partial \theta}$ and $\frac{\partial \ell^{R}}{\partial \theta}$ can evaluate the difference of the two gradients. Accordingly, we pre-define a threshold $\epsilon$ to determine if the difference is considered as minor and the update with $\frac{\partial \ell^{S}}{\partial \theta}$ will be close to both $\theta^{*(S)}$ and $\theta^{*(T)}$. If the difference is significant, the proposed reference process will adjust $\frac{\partial \ell^{S}}{\partial \theta}$ so that it will move more towards $\theta^{*(T)}$. This process is defined as follows
\begin{align}
    \tilde{g}=\begin{cases} \argmax_{g} \cos(g, \frac{\partial \ell^{R}}{\partial \theta})-\lambda \|g\|^{2}_{2}|_{g_{0}=\frac{\partial \ell^{S}}{\partial \theta}} & \text{if} \cos(\frac{\partial \ell^{S}}{\partial \theta}, \frac{\partial \ell^{R}}{\partial \theta}) < \epsilon, \\
    \frac{\partial \ell^{S}}{\partial \theta} & \text{otherwise},
              \end{cases}
\label{eqn:refer}
\end{align}
where $\lambda$ is the regularization parameter and $\cos(\cdot, \cdot)$ is the cosine similarity, \ie $\cos(a,b)=a^{\top}b/|a||b|$ ($a$ and $b$ are the input vectors), and $\tilde{g}$ is the output gradient. The embedded optimization problem in \eqn (\ref{eqn:refer}) aims to find a $\tilde{g}$, which is with an initial point $g_{0}=\frac{\partial \ell^{S}}{\partial \theta}$, to be consistent with the reference gradient $\frac{\partial \ell^{R}}{\partial \theta}$ in terms of cosine similarity. In other words, the reference gradient $\frac{\partial \ell^{R}}{\partial \theta}$ provides a reference so that $\tilde{g}$ is able to be aware of a rough direction towards the underlying $\theta^{*(T)}$. In this way, the knowledge learned from $D^{S}$ is transferred to the target domain.
We solve the embedded optimization problem with the gradient ascent method because our goal is to maximize the cosine similarity between $\tilde{g}$ and $\frac{\partial \ell^{R}}{\partial \theta}$.
Subsequently, $\theta$ would be updated with $\tilde{g}$, \ie $\theta \leftarrow \theta - \eta\tilde{g}$, where $\eta$ is a learning rate.
Note that $\frac{\partial \ell^{S}}{\partial \theta}$ is generated by randomly selected training samples and is the initial point for $\tilde{g}$. As a result, the process of optimizing cosine similarity in the training stage is almost surely stochastic. This can effectively prevent $\tilde{g}$ from overfitting $\frac{\partial \ell^{R}}{\partial \theta}$.

The proposed reference process yields $\tilde{g}$ to update the model so that the parameters are close to the underlying $\theta^{*(T)}$. As $\theta$ is learned with the references from the target domain, by the chain rule, $\frac{\partial \ell^{S}}{\partial \theta_{body}}=\frac{\partial \ell^{S}}{\partial x^{S}} \frac{\partial x^{S}}{\partial \theta_{body}}$ and $\frac{\partial \ell^{S}}{\partial x^{S}}$ can be considered as a function of $\theta$. So $\theta_\mathsf{b}$ will be affected by the references as well.

As the number of references is expected to be far smaller than the training data, we follow a similar idea of the stochastic process to randomly draw a reference from the reference pool at each iteration. 
\section{Experiments}
\label{sec:experiment}

In this section, we introduce the experimental protocol, present the experimental results, and then have a discussion about the results.

\subsection{Experimental Setup}
\label{subsec:setup}

\subsubsection{Datasets.}

\noindent We adopt the large-scale saliency prediction dataset SALICON~\cite{Jiang_CVPR_2015} (the 2017 version) and the MIT1003 \cite{Judd_ICCV_2009} as the source domain datasets. 
Accordingly, we adopt WebSal~\cite{Shen_ECCV_2014} and the art subset in CAT2000~\cite{Borji_arXiv_2015} as the target domain datasets. 
Specifically, SALICON consists of 10000 real-world images, MIT1003 consists of 1003 natural scene images, and WebSal consists of 149 webpage screenshots.
CAT2000 includes 20 categories and each category has 100 images. Art is one of the most common categories, whose images are the pictures of human-made works, like the paintings, handcrafts, and \etc

\subsubsection{Baseline Models.}

To study how well the proposed method would generalize to different models, we use two baseline models, \ie~DINet~\cite{Yang_TMM_2012} and ResNet-50~\cite{He_CVPR_2016}.

\subsubsection{Settings.}

There are three dimensions to the experiments in this work, \ie~source domain samples, baseline model, and target domain samples. 
Specifically, the baseline model is trained with the source domain samples. 
The learned model is further fine-tuned with the target domain samples. This setting is similar in the case of the proposed method. 
For convenience, we denote the setting as a combination of the initials of the datasets or the models, \eg~$\langle \mathsf{S}, \mathsf{D}, \mathsf{W} \rangle$ indicates that we use SALICON as the source domain dataset, DINet as the baseline model, and WebSal as the target domain dataset. 
Similarly, we use initials $\mathsf{M}$, $\mathsf{R}$, and $\mathsf{A}$ to represent MIT1003, ResNet-50, and Art, respectively.

To understand how the number of references affects the performance, we evaluate the proposed method with $n={1,5,10}$. 
Moreover, to provide a benchmark of the performance w.r.t. more references, a paradigm that is similar to 3-fold cross validation is applied with more references. 
For instance, given WebSal as the target domain datasets, we divide it into three subsets, which contain 50, 50, and 49 images, respectively. Then, we alternately use any two subsets as the reference samples and the rest as the validation set. The process is repeated 3 times. We denote the results of this process as an empirical upper bound.

\subsubsection{Evaluation Metrics.}

We adopt the common metrics used in \cite{Bylinskii_PAMI_2018} and \cite{Jiang_CVPR_2015}, \ie normalized scanpath saliency (NSS)~\cite{Itti_ASNN_2003,Rothenstein_IVC_2008}, area under curve (AUC)~\cite{Borji_TIP_2013,Judd_Report_2012}, and correlation coefficient (CC)~\cite{Ouerhani_ELCVIAs_2004}. Higher scores indicate better performance. We use the public implementation\footnote{\url{https://github.com/NUS-VIP/salicon-evaluation}} provided by \cite{Jiang_CVPR_2015}.
Each experiment is repeated 10 times and the mean metric scores are reported. Due to the space limits, we report the corresponding standard deviation in the supplementary document.

\subsection{Training Scheme}
\label{subsec:exp}

We follow the widely-used two-stage transfer learning framework~\cite{Bauml_ICRA_2019,Guo_CVPR_2019,Shan_ICISBDE_2017,Li_ICML_2018}, \ie first train a model with the source domain data and fine-tune with the target domain data. 
We denote the trained model as $\mathsf{TR}$ and the fine-tuned model as $\mathsf{FT}$. 
In the proposed framework, 
the $n$-reference training stage first trains a model with the source domain data and $n$ target domain references (denoted as $\mathsf{TR\!-\!Ref}$), 
and then further fine-tune with the references (denoted as $\mathsf{FT|Ref}$). 

Regarding the experimental details, we follow DINet~\cite{Yang_arXiv_2019} to use Adam optimizer \cite{Kingma_arXiv_2014} with learning rate $\eta$ = 5e-5 and weight decay 1e-4. 
We use batch size 10 for all the experiments. The number of epochs is 10 and we decrease the learning rate for every 3 epochs by multiplying with 0.2. 
In $\mathsf{TR\!-\!Ref}$, we randomly sample 10 training data without replacement as the training sample at each iteration. 
Meanwhile, we randomly sample $n_{r}$ references with replacement as the reference. 
In this way, the difference between the number of training samples and references will not cause a problem. $n_{r}$ are 1, 3 and 5 in the experiments with $n=1,5,10$, respectively.
This process is the same for the one of $\mathsf{FT}$. 
We select the model with the best performance over epochs for further fine-tuning. The normalized $l_{1}$ loss~\cite{Yang_arXiv_2019} is used and the threshold $\epsilon$ is set to 0 for all the experiments.
We implement the proposed framework with PyTorch~\cite{Paszke_NIPSW_2017}. 

\setlength{\tabcolsep}{4pt}
\begin{table}[!t]
	\begin{center}
		\caption{
			Performance with various settings of $\langle \mathsf{source}, \mathsf{model}, \mathsf{target} \rangle$.
			Here, $\mathsf{S}$ is SALICON, $\mathsf{M}$ is MIT1003, $\mathsf{W}$ is WebSal, $\mathsf{A}$ is Art subset, $\mathsf{D}$ is DINet, and $\mathsf{R}$ is ResNet.
		    $\uparrow$ implies that a higher score is better. 
		    The score in bold font indicates the best result under the respective metric. 
		    We report the mean score from 10 runs for conventional training (\ie~$n=0$) and the proposed method. 
			The empirical upper bound (EUB) is generated by 3-fold cross validation on the target domain. 
			The experimental details are provided in Section~\ref{subsec:setup} and \ref{subsec:exp}
		}
		\label{tbl:s_d_w}
		\begin{tabular}{lccccccc}
			\hline\noalign{\smallskip}
			&  & \multicolumn{3}{c}{$\langle \mathsf{S}, \mathsf{D}, \mathsf{W} \rangle$} & \multicolumn{3}{c}{$\langle \mathsf{S}, \mathsf{R}, \mathsf{W} \rangle$} \\ 	 \cmidrule(lr){3-5} \cmidrule(lr){6-8}
			& & NSS$\uparrow$ & AUC$\uparrow$ & CC$\uparrow$ & NSS$\uparrow$ & AUC$\uparrow$ & CC$\uparrow$ \\ \cmidrule(lr){3-5} \cmidrule(lr){6-8}
			$\mathsf{FT}$ w/o $\mathsf{TR}$~~ & ~$n=10$         & 0.8252 & 0.7430 & 0.3635 & 0.8846 & 0.7455 & 0.3852 \\ \hline\noalign{\smallskip}
			$\mathsf{TR}$          & ~~~~$n=0$~~~~         & 1.3330 & 0.7796 & 0.5515 & 1.2950 & 0.7749 & 0.5358 \\ \hline\noalign{\smallskip}
			$\mathsf{TR\!-\!Ref}$  & $n=1$    & 1.3621 & 0.7848 & 0.5628 & 1.3569 & 0.7864 & 0.5611 \\ 
			$\mathsf{FT}$          & $n=1$    & 1.4731 & 0.8005 & 0.5976 & 1.3722 & 0.7923 & 0.5627 \\
			$\mathsf{FT|Ref}$  & $n=1$        & \textbf{1.5077} & \textbf{0.8051} & \textbf{0.6121} & \textbf{1.4272} & \textbf{0.7983} & \textbf{0.5817} \\ \hline\noalign{\smallskip}
			$\mathsf{TR\!-\!Ref}$  & $n=5$    & 1.3683 & 0.7874 & 0.5659 & 1.3535 & 0.7837 & 0.5593 \\
			$\mathsf{FT}$          & $n=5$    & 1.5803 & 0.8161 & 0.6355 & 1.5043 & 0.8131 & 0.6139 \\
			$\mathsf{FT|Ref}$      & $n=5$    & \textbf{1.6085} & \textbf{0.8200} & \textbf{0.6468} & \textbf{1.5491} & \textbf{0.8149} & \textbf{0.6281} \\ \hline\noalign{\smallskip}
			$\mathsf{TR\!-\!Ref}$  & ~$n=10$  & 1.3647 & 0.7839 & 0.5633 & 1.3583 & 0.7857 & 0.5612 \\ 
			$\mathsf{FT}$          & ~$n=10$  & 1.6290 & 0.8247 & 0.6531 & 1.5164 & 0.8103 & 0.6200 \\
			$\mathsf{FT|Ref}$      & ~$n=10$  & \textbf{1.6439} & \textbf{0.8276} & \textbf{0.6605} & \textbf{1.5829} & \textbf{0.8143} & \textbf{0.6414} \\ \hline\noalign{\smallskip}
			$\mathsf{TR\!-\!Ref}$  & EUB      & 1.3822 & 0.7910 & 0.5708 & 1.3626 & 0.7864 & 0.5645 \\
			$\mathsf{FT}$          & EUB      & 1.8695 & 0.8488 & 0.7389 & 1.8325 & 0.8462 & 0.7275 \\
			$\mathsf{FT|Ref}$      & EUB      & \textbf{1.8831} & \textbf{0.8494} & \textbf{0.7442} & \textbf{1.8500} & \textbf{0.8480} & \textbf{0.7321} \\
			\hline\noalign{\smallskip}
			& & \multicolumn{3}{c}{$\langle \mathsf{M}, \mathsf{D}, \mathsf{W} \rangle$} & \multicolumn{3}{c}{$\langle \mathsf{S}, \mathsf{D}, \mathsf{A} \rangle$} \\ 
			\cmidrule(lr){3-5} \cmidrule(lr){6-8}
			& & NSS$\uparrow$ & AUC$\uparrow$ & CC$\uparrow$ & NSS$\uparrow$ & AUC$\uparrow$ & CC$\uparrow$ \\
			\cmidrule(lr){3-5} \cmidrule(lr){6-8}
			$\mathsf{FT}$ w/o $\mathsf{TR}$         & ~$n=10$         & 0.8252 & 0.7430 & 0.3635 & 1.2183 & 0.8339 & 0.5161 \\ \hline\noalign{\smallskip}
			$\mathsf{TR}$          & $n=0$         & 1.3905 & 0.7991 & 0.5700 & 1.5172 & 0.8225 & 0.6003 \\ \hline\noalign{\smallskip}
			$\mathsf{TR\!-\!Ref}$  & $n=1$         & 1.4405 & 0.8085 & 0.5902 & 1.5651 & 0.8287 & 0.6211 \\
			$\mathsf{FT}$          & $n=1$         & 1.4410 & 0.8023 & 0.5784 & 1.6255 & 0.8324 & 0.6449 \\
			$\mathsf{FT|Ref}$~~~~  & $n=1$         & \textbf{1.4575} & \textbf{0.8070} & \textbf{0.5838} & \textbf{1.6523} & \textbf{0.8380} & \textbf{0.6564} \\
			\hline\noalign{\smallskip} 
			$\mathsf{TR\!-\!Ref}$  & $n=5$         & 1.4452 & 0.8064 & 0.5908 & 1.5870 & 0.8304 & 0.6274 \\
			$\mathsf{FT}$          & $n=5$         & 1.5795 & 0.8217 & 0.6395 & 1.8049 & 0.8480 & 0.7185 \\
			$\mathsf{FT|Ref}$~~~~  & $n=5$         & \textbf{1.6136} & \textbf{0.8269} & \textbf{0.6515} & \textbf{1.8314} & \textbf{0.8503} & \textbf{0.7274} \\
			\hline\noalign{\smallskip}
			$\mathsf{TR\!-\!Ref}$  & ~$n=10$         & 1.4330 & 0.8060 & 0.5872 & 1.5704 & 0.8288 & 0.6204 \\
			$\mathsf{FT}$          & ~$n=10$         & 1.6462 & 0.8261 & 0.6660 & 1.8325 & 0.8474 & 0.7288 \\
			$\mathsf{FT|Ref}$~~~~  & ~$n=10$         & \textbf{1.6691} & \textbf{0.8283} & \textbf{0.6730} & \textbf{1.8584} & \textbf{0.8503} & \textbf{0.7366} \\
			\hline\noalign{\smallskip}
			$\mathsf{TR\!-\!Ref}$  & EUB         & 1.4402 & 0.8087 & 0.5905 & 1.5980 & 0.8340 & 0.6331 \\
			$\mathsf{FT}$          & EUB         & 1.8450 & 0.8466 & 0.7330 & 2.1595 & 0.8636 & 0.8464 \\
			$\mathsf{FT|Ref}$~~~~  & EUB         & \textbf{1.8507} & \textbf{0.8478} & \textbf{0.7344} & \textbf{2.1874} & \textbf{0.8649} & \textbf{0.8519} \\
			\hline
		\end{tabular}
	\end{center}
\end{table}
\setlength{\tabcolsep}{1.4pt}

\subsection{Performance}

The experimental results with the following settings, \ie~$\langle \mathsf{S}, \mathsf{D}, \mathsf{W} \rangle$, $\langle \mathsf{S}, \mathsf{R}, \mathsf{W} \rangle$, $\langle \mathsf{M}, \mathsf{D}, \mathsf{W} \rangle$, and $\langle \mathsf{S}, \mathsf{D}, \mathsf{A} \rangle$, are shown in \tab~\ref{tbl:s_d_w}.
Within setting $\langle \mathsf{S}, \mathsf{D}, \mathsf{W} \rangle$, the proposed framework (\ie $\mathsf{FT|Ref}$) achieves better performance than $\mathsf{FT}$ over all metrics. 
Particularly, as the number of references increases, the consequently trained models provide better initializations for fine-tuning. 
In other words, $\mathsf{FT|Ref}$ yields better performance when the dependent trained model uses more reference samples.
Using a different baseline model,
we experiment it with setting $\langle \mathsf{S}, \mathsf{R}, \mathsf{W} \rangle$ which $\mathsf{FT|Ref}$ achieves consistent improvement. 
Moreover, using DINet as the baseline model leads to better performance than using ResNet-50. 

We study how well the proposed framework generalizes to different target domain data using setting $\langle \mathsf{S}, \mathsf{D}, \mathsf{A} \rangle$. 
As seen in \tab~\ref{tbl:s_d_w}, similar performance improvement can be found, which implies the proposed framework can generalize to a different target domain.
Furthermore, the study with MIT1003 as the source domain dataset, \ie~setting $\langle \mathsf{M}, \mathsf{D}, \mathsf{W} \rangle$, shows consistent improvement.
The overall performance within setting $\langle \mathsf{M}, \mathsf{D}, \mathsf{W} \rangle$ is slightly lower than the one within setting $\langle \mathsf{S}, \mathsf{D}, \mathsf{W} \rangle$. 
This implies that SALICON is more efficient than MIT1003 to transfer the knowledge to WebSal.
On the other hand, models trained with one sample in target domain have noticeable gaps w.r.t. EUB, and is improved with more training samples.
This is consistent with the implication of Theorem~\ref{thrm:gb}.

We perform paired t-test and permutation test over images within setting $\langle \mathsf{S}, \mathsf{D}, \mathsf{W} \rangle$ to evaluate the difference between $\mathsf{TR\!-\!Ref}$ and $\mathsf{FT|Ref}$. 
Both corresponding $p$ are less than 0.001. 
This implies that $\mathsf{TR\!-\!Ref}$ significantly provides a good initialization to $\mathsf{FT|Ref}$ to yield high performance.
To validate the effect of knowledge transfer in saliency prediction, 
we conduct the experiment where models are learned using only the target domain samples, \ie~$\mathsf{FT}$ w/o $\mathsf{TR}$ in \tab~\ref{tbl:s_d_w}. 
We set $n=10$ as $n=1,5$ will yield much worse performance. 
In  all settings, the performance of $\mathsf{FT}$ w/o $\mathsf{TR}$ significantly drops when compare to $\mathsf{FT|Ref}$.
These results are even lower than $\mathsf{TR}$ and $\mathsf{FT}$,
which indicate the importance of efficient initialization with a source domain dataset.
\section{Analysis}

We study the influences of the number of references, the threshold $\epsilon$, and the layers updated by the proposed framework. 
All analysis are within setting $\langle \mathsf{S}, \mathsf{D}, \mathsf{W} \rangle$, where the mean score and standard deviation from 3 runs are reported.

\begin{figure}[!t]
	\centering
	\subfloat[] {\includegraphics[width=0.35\linewidth]{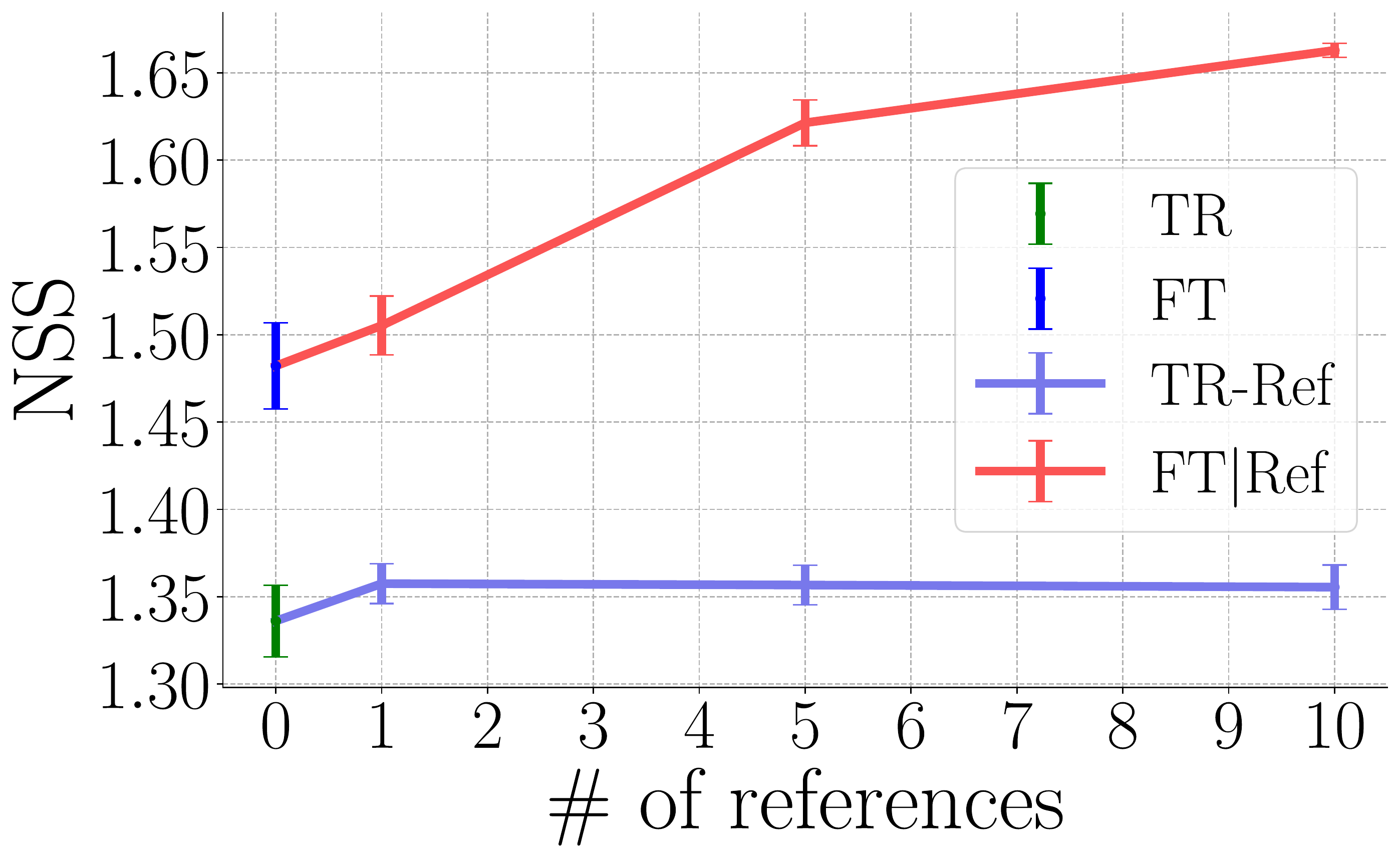} \label{fig:abla_nss_shots}} \hspace{7ex}
	\subfloat[] {\includegraphics[width=0.35\linewidth]{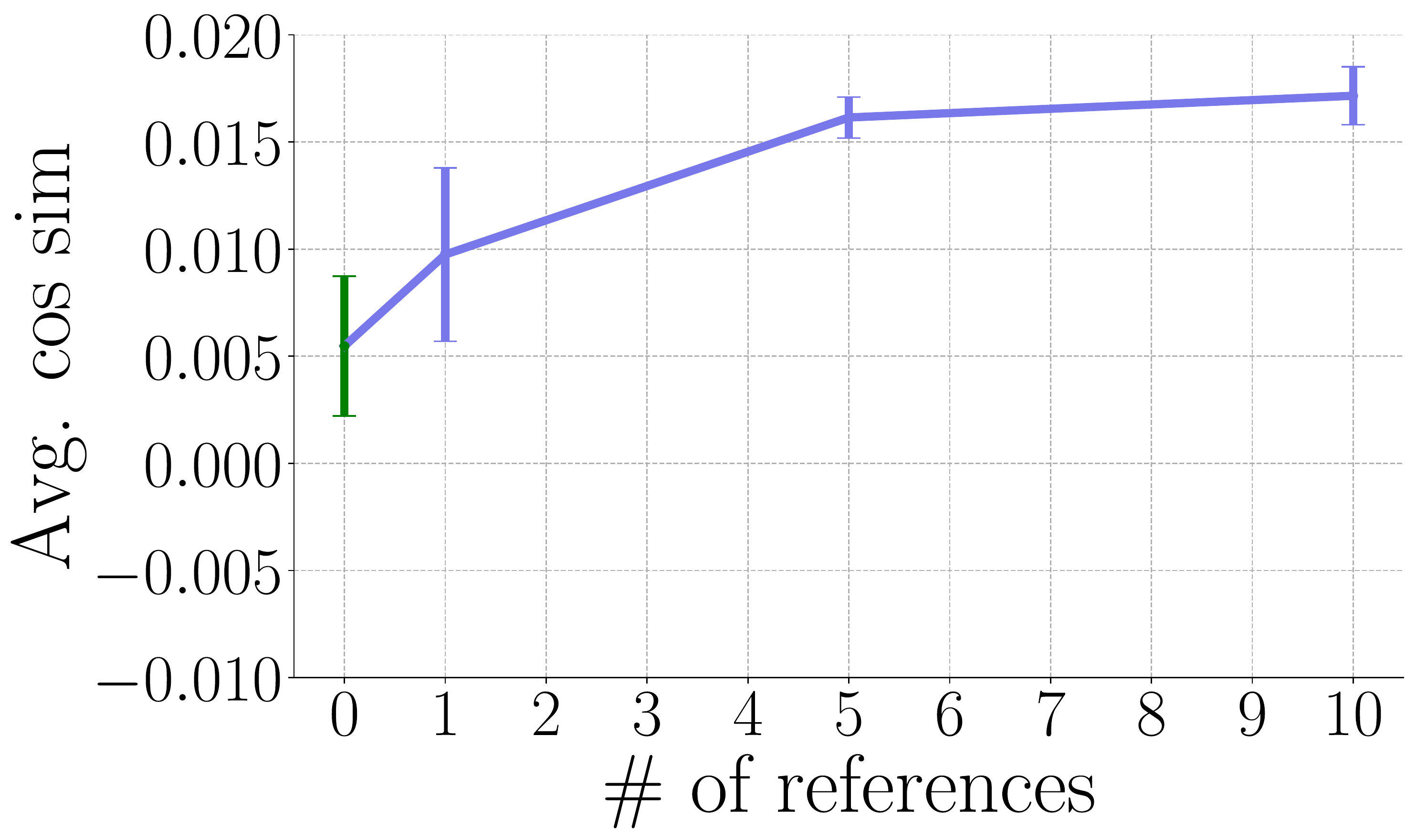} \label{fig:abla_cos_shots}} \\
	\subfloat[] {\includegraphics[width=0.35\linewidth]{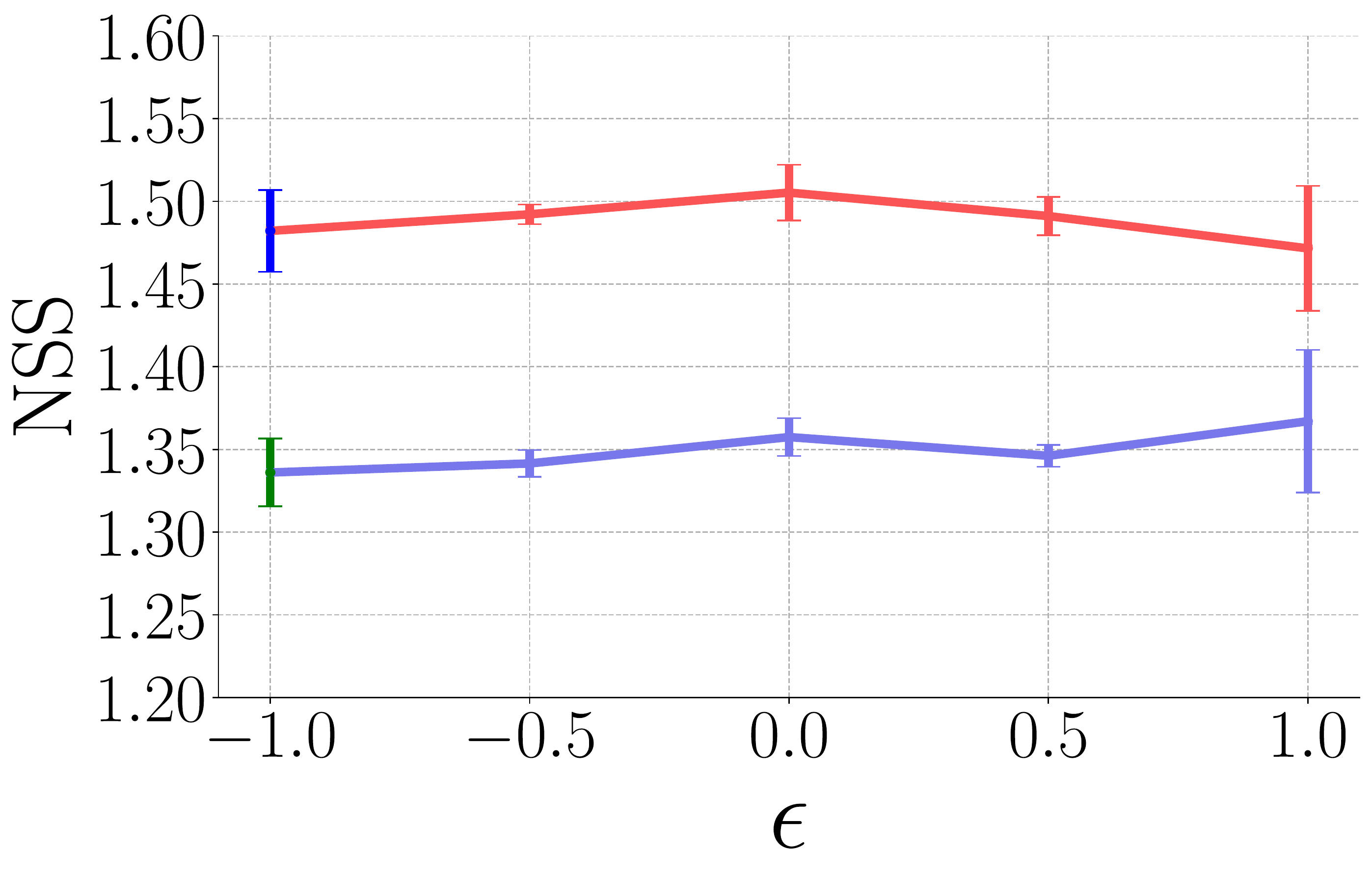}    \label{fig:abla_nss_th}}   \hspace{7ex}
	\subfloat[] {\includegraphics[width=0.35\linewidth]{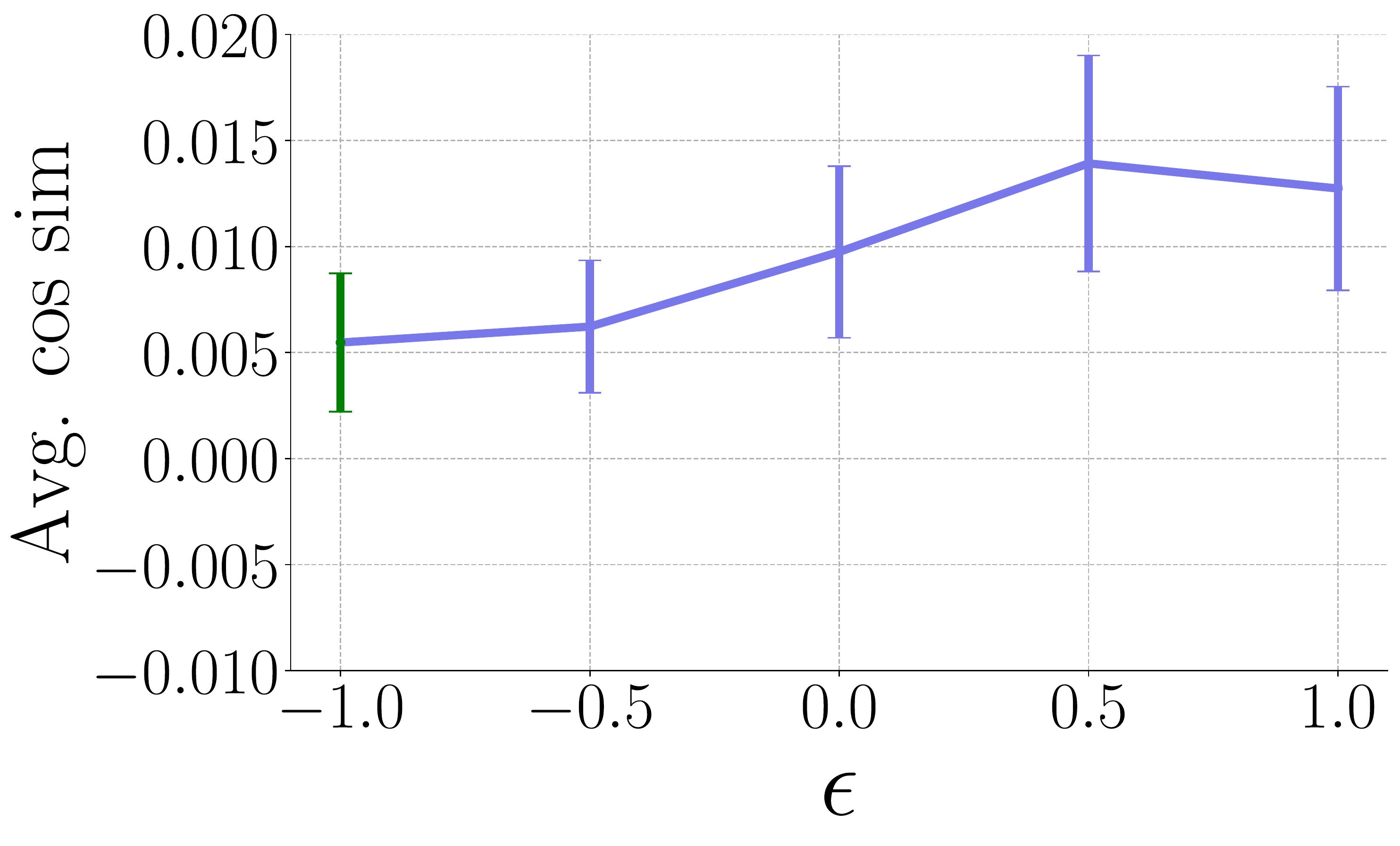}    \label{fig:abla_cos_th}}
	\caption{The effect of the number of references (a, b) and threshold $\epsilon$ (c, d) on NSS metric and average cosine similarity within setting $\langle \mathsf{S}, \mathsf{D}, \mathsf{W} \rangle$.
		$n=0$ indicates that no reference sample is used.
		Hence, $\mathsf{TR\!-\!Ref}$ and $\mathsf{FT|Ref}$ turn to be $\mathsf{TR}$ and $\mathsf{FT}$. 
		The results of c and d are generated with $n=1$.
		$\epsilon$ determines whether the gradient needs to be corrected or not (see \fig~\ref{fig:reference}).
		Comparing to $\mathsf{TR}$, $\mathsf{FT}$, and $\mathsf{FT|Ref}$, only $\mathsf{TR\!-\!Ref}$ is able to evaluate the cosine similarity between the samples from the source domain and target domain (see \fig~\ref{fig:framework})
	}
	\label{fig:eff_shots}
\end{figure}

\subsection{Ablation Study}

\subsubsection{Effect of Number of References.}

As shown in \fig~\ref{fig:eff_shots}, as the number of references increases, the performance of $\mathsf{TR\!-\!Ref}$ keeps flat or even slightly drops, but the performance of $\mathsf{FT|Ref}$ is significantly improved. This implies that the proposed reference process with more reference samples can yield better initialization for fine-tuning. Moreover, the average cosine similarity is increased with more references. This implies that the number of references is helpful to adapt the training process with source domain data to the target domain data.

\subsubsection{Effect of Threshold $\epsilon$.}

We experiment with the proposed framework with $n=1$, which is more representative and challenging than cases with more references, with various thresholds. An interesting observation in \fig~\ref{fig:abla_nss_th} is that although $\epsilon=1$ achieves best performance on $\mathsf{TR\!-\!Ref}$, it deteriorates the performance of $\mathsf{FT|Ref}$. This shows that when $\epsilon=1$, all the gradients at each iteration need to be corrected because the cosine similarity between any two gradients is equal or less than 1. As a result, the reference process enforces the training process to overfit the reference samples. This can be verified in \fig~\ref{fig:abla_cos_th} where the average cosine similarity is roughly increased as $\epsilon$ is increasing.

\begin{figure}[!t]
	\centering
	\subfloat[] {\includegraphics[width=.28\linewidth]{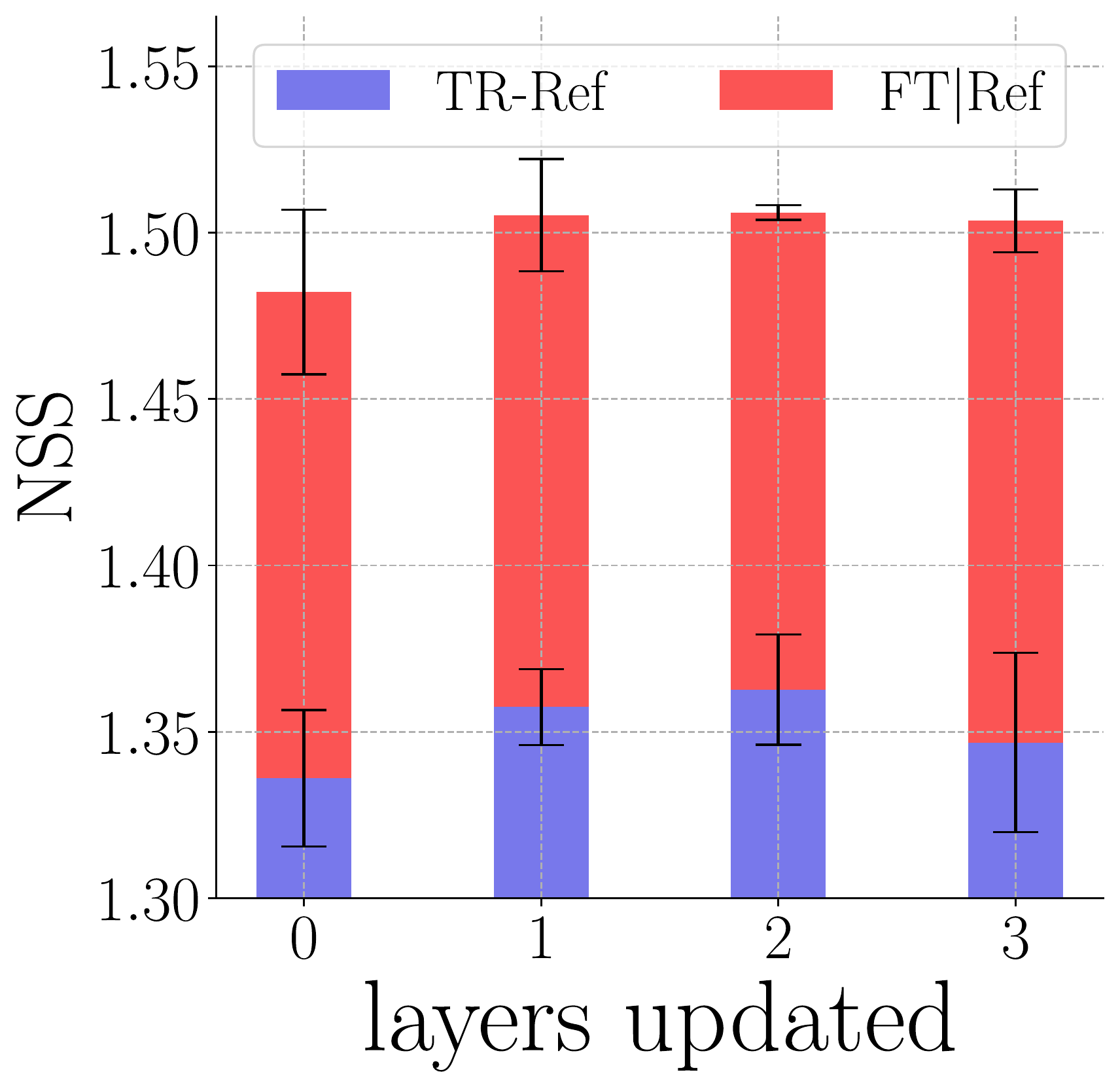}     \label{fig:abla_nss_layer}} \hfill
	\subfloat[] {\includegraphics[width=0.28\linewidth]{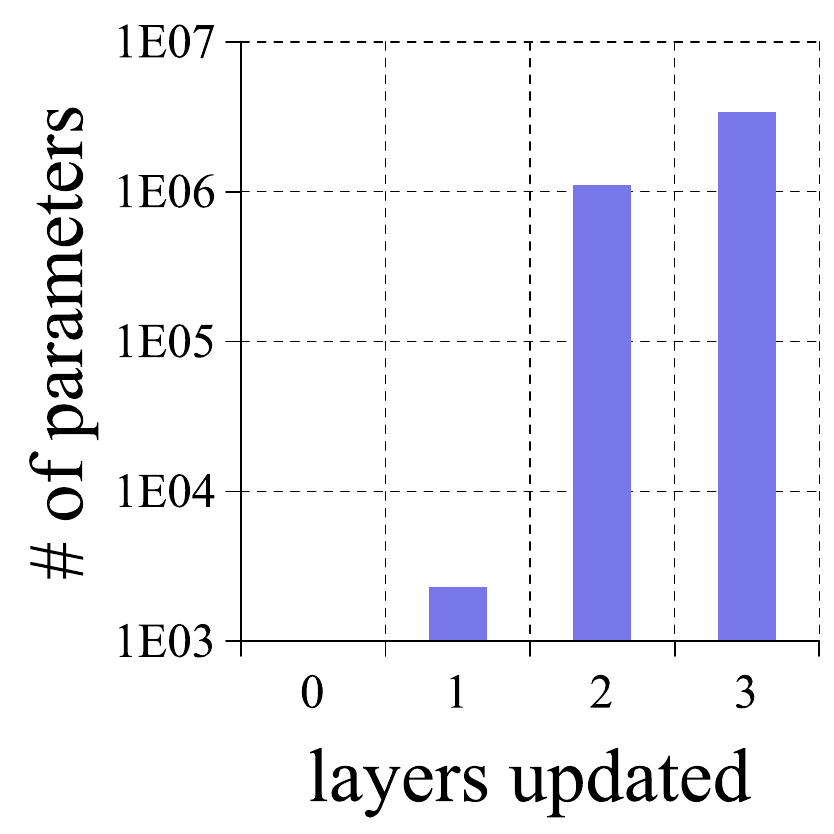} \label{fig:abla_parameters}} \hfill
	\subfloat[] {\includegraphics[width=0.28\linewidth]{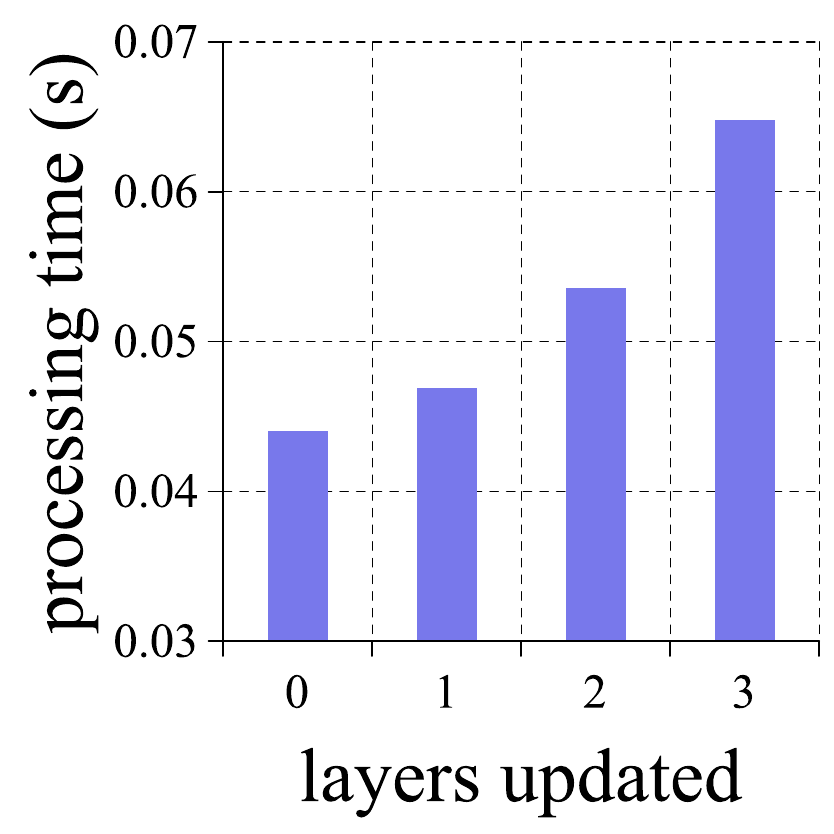}        \label{fig:abla_para_time}}
	\caption{Ablation study of downstream layers updated within setting $\langle \mathsf{S}, \mathsf{D}, \mathsf{W} \rangle$ with $n=1$. Note that when 0 layer is updated, it turns to be $\mathsf{TR}$ and $\mathsf{FT}$}
	\label{fig:para_time}
\end{figure}

\subsubsection{Effect of Updated Layers.}

To understand the effect of layers updated by the proposed $1$-reference transfer learning, we experiment with various downstream layers. Consequently, the performance is shown in \fig~\ref{fig:abla_nss_layer}, while the number of parameters and the computational cost are reported in \fig \ref{fig:abla_parameters} and \fig~\ref{fig:abla_para_time}, respectively. The layers are downstream layers, which are close to the output. When 0 layer is updated, $\mathsf{TR\!-\!Ref}$ and $\mathsf{FT|Ref}$ are equivalent to $\mathsf{TR}$ and $\mathsf{FT}$, respectively. The baseline model in this experiment is DINet.

\fig~\ref{fig:abla_nss_layer} shows that using the last 2 layers achieves slightly better performance in NSS than using the other numbers of the last layers. However, it takes 69 milliseconds longer in the training process than using the last layers. In light of the trade-off, we use the last layer of the baseline model in Section~\ref{sec:experiment}.

\begin{figure*}[!t]
	\centering
	\includegraphics[width=1.0\textwidth]{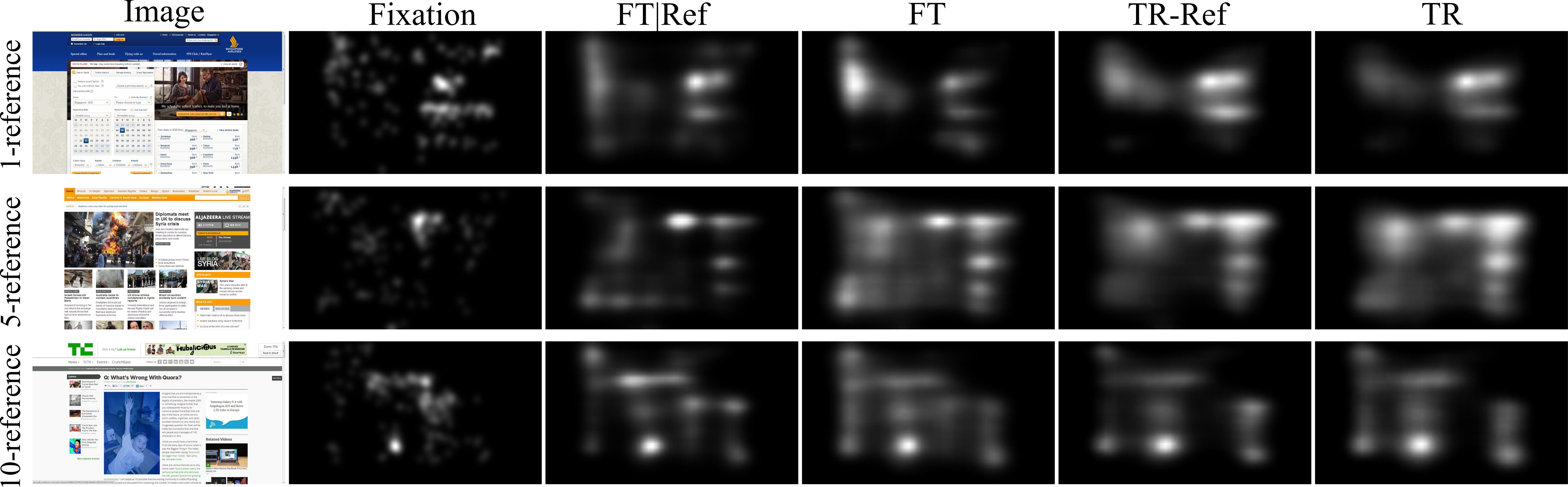} 
	\caption{Qualitative results with human fixations and maps generated by the models trained by the four procedures}
	\label{fig:qual}
\end{figure*}

\begin{figure*}[!t]
	\centering
	\includegraphics[width=1.0\textwidth]{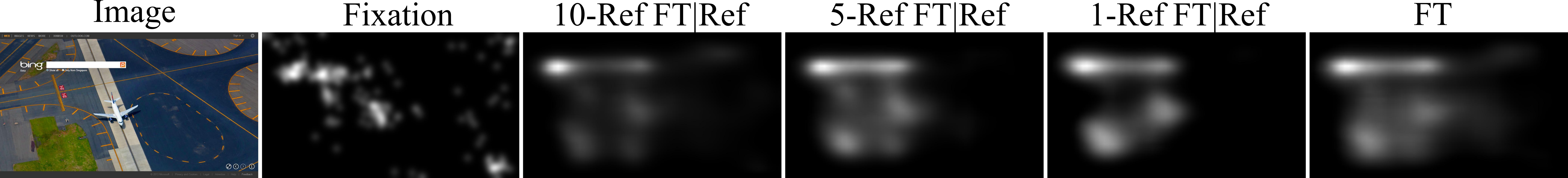} 
	\caption{Qualitative results w.r.t. different $n$}
	\label{fig:qual_n}
\end{figure*}

\subsection{Qualitative Comparison}

\fig~\ref{fig:qual} shows the comparison between the predicted saliency maps generated by $\mathsf{TR}$, $\mathsf{TR\!-\!Ref}$, $\mathsf{FT}$, and $\mathsf{FT|Ref}$. It can be observed that with the reference process, the proposed framework efficiently leverages the knowledge learned from the source domain, which are based on natural scene images, to subtly identify salience in the new domain. 
Taking the example in the first row, $\mathsf{FT|Ref}$ predicts that the people is salient, which takes the learned knowledge into account, whereas $\mathsf{FT}$ predicts that the people is less salient than the text.
\fig~\ref{fig:qual_n} shows more references lead to better prediction.

\section{Conclusion}

This work studies how to leverage the knowledge learned from a source domain that has adequate images and corresponding human fixations and very few samples (\ie~references) from a new domain (\ie~target domain) to predict saliency maps in the target domain. 
We propose an $n$-reference transfer learning framework to guide the training process to converge to a local minimum in favor of the target domain. 
The proposed framework is gradient-based and model-agnostic. 
Comprehensive experiments and ablation studies to evaluate the proposed framework are reported. 
Results show the effectiveness of the framework with a significant performance improvement.
\subsubsection{Acknowledgments.} This research was funded in part by the NSF under Grants 1908711, 1849107, in part by the University of Minnesota Department of Computer Science and Engineering Start-up Fund (QZ), and in part supported by the National Research Foundation, Singapore under its Strategic Capability Research Centres Funding Initiative. Any opinions, findings and conclusions or recommendations expressed in this material are those of the author(s) and do not reflect the views of National Research Foundation, Singapore.

\clearpage
%
%

\end{document}


\pagestyle{headings}
\mainmatter
\def\ECCVSubNumber{615}  

\title{Supplementary File for $n$-Reference Transfer Learning for Saliency Prediction} 

\titlerunning{$n$-Reference Transfer Learning for Saliency Prediction}
%
\author{Yan~Luo\inst{1}\orcidID{0000-0001-5135-0316} \and
Yongkang~Wong\inst{2}\orcidID{0000-0002-1239-4428} \and
Mohan~S.~Kankanhalli\inst{2}\orcidID{0000-0002-4846-2015} \and
Qi~Zhao\inst{1}\orcidID{0000-0003-3054-8934}}
\authorrunning{Y. Luo et al.}
%
\institute{Department of Computer Science and Engineering, University of Minnesota \\
\email{luoxx648@umn.edu}, \email{qzhao@cs.umn.edu}\\
\and
School of Computing, National University of Singapore \\
\email{\{wongyk, mohan\}@comp.nus.edu.sg}}

\maketitle
In this supplementary document, we first present the proof of Theorem 1. Then, we provide both the mean and standard deviation of scores on NSS, AUC, and CC metrics within the four settings, \ie $\langle$SALICON, DINet, WebSal$\rangle$ (see \tab \ref{tbl:s_d_w_sup}), $\langle$SALICON, ResNet-50, WebSal$\rangle$ (see \tab \ref{tbl:s_r_w_sup}), $\langle$MIT1003, DINet, WebSal$\rangle$ (see \tab \ref{tbl:m_d_w_sup}), and $\langle$SALICON, DINet, Art$\rangle$ (see \tab \ref{tbl:s_d_a_sup}). Overall, the standard deviations generated by the proposed $\mathsf{TR\!-\!Ref}$ and $\mathsf{FT|Ref}$ are similar to the ones generated by $\mathsf{TR}$ and $\mathsf{FT}$, respectively.

\begin{theorem}[Saliency generalization bound]
	Denote $H$ as a finite hypothesis set. Given $\ell^{p}$ and $y\in [0,1]^{m}$, for any $\delta>0$, with probability at least $1-\delta$, the following inequality holds for all $f\in H$:
	\begin{align*}
	|R_{\mathcal{D}}(f) - \hat{R}_{\mathcal{D}}(f)| \le  m^{\frac{1}{p}}\sqrt{\frac{\log|H|+\log\frac{2}{\delta}}{2|D|}}
	\end{align*}
\end{theorem}
\begin{proof}[]
	The proof sketch is similar to the regression generalization bound provided in \cite{Mohri_MIT_2012}.
	First, as $\ell^{p}(y_{1},y_{2})=(\sum_{i}^{m}|y_{1i}-y_{2i}|^p)^{\frac{1}{p}}\le m^{\frac{1}{p}}$, we know $\ell^{p}$ is bounded by $m^{\frac{1}{p}}$. Then, by the union bound, given an error $\xi$, we have
	\begin{align*}
	Pr[\sup_{f\in H}|R(f)-\hat{R}(f)| > \xi] \le \sum_{f\in H}^{} Pr[|R(f)-\hat{R}(f)|> \xi].
	\end{align*}
	By Hoeffding's bound, we have
	\begin{align*}
	\sum_{f\in H}^{} Pr[|R(f)-\hat{R}(f)|> \xi] \le 2|H|\exp (-\frac{2|D|\xi^2}{m^{\frac{2}{p}}}).
	\end{align*}
	Due to the probability definition, $2|H|\exp (-\frac{2|D|\xi^2}{m^{\frac{2}{p}}}) = \delta$. Considering $\xi$ is a function of other variables, we can rearrange it as 
	$\xi=m^{\frac{1}{p}}\sqrt{\frac{\log|H|+\log\frac{2}{\delta}}{2|D|}}$.
	Since we know $Pr[|R(f)-\hat{R}(f)| > \xi]$ is with probability at most $\delta$, it can be inferred that $Pr[|R(f)-\hat{R}(f)| <= \xi]$ is at least $1-\delta$.
	\qed
\end{proof}

\newpage

\begin{table}[!t]
	\centering
	\caption{Performance within setting $\langle$SALICON, DINet, WebSal$\rangle$. $\uparrow$ implies that higher score is better. The score in bold font indicates the best result under the metric. We take 10 runs for conventional training, 1-shot, 5-shot, and 10-shot transfer learning and report the mean and the std. Empirical upper bound (EUB) is generated by the 3-fold cross validation on the target domain dataset. The details of the experimental setup are provided in Section~\ref{subsec:setup}}
	\footnotesize
	\begin{tabular}{lcccc}
		\toprule
		& & NSS$\uparrow$ & AUC$\uparrow$ & CC$\uparrow$ \\
		\cmidrule(lr){3-3} \cmidrule(lr){4-4} \cmidrule(lr){5-5}
		$\mathsf{TR}$ & $n=0$ & ~~1.3330$\pm$0.0084~~ & ~~0.7796$\pm$0.0025~~ & ~~0.5515$\pm$0.0033~~  \\ \midrule
		$\mathsf{TR\!-\!Ref}$~~~~~~~ & $n=1$ & 1.3621$\pm$0.0191 & 0.7848$\pm$0.0023 & 0.5628$\pm$0.0073 \\ 
		$\mathsf{FT}$ & $n=1$ & 1.4731$\pm$0.0466 & 0.8005$\pm$0.0120 & 0.5976$\pm$0.0175 \\
		$\mathsf{FT|Ref}$ & $n=1$ & \textbf{1.5077}$\pm$0.0497 & \textbf{0.8051}$\pm$0.0121 & \textbf{0.6121}$\pm$0.0171 \\ \midrule
		
		$\mathsf{TR\!-\!Ref}$ & $n=5$ & 1.3683$\pm$0.0266 & 0.7874$\pm$0.0083 & 0.5659$\pm$0.0124  \\
		$\mathsf{FT}$ & $n=5$ & 1.5803$\pm$0.0346 & 0.8161$\pm$0.0062 & 0.6355$\pm$0.0113  \\
		$\mathsf{FT|Ref}$ & $n=5$ & \textbf{1.6085}$\pm$0.0212 & \textbf{0.8200}$\pm$0.0050 & \textbf{0.6468}$\pm$0.0084 \\ \midrule
		
		$\mathsf{TR\!-\!Ref}$ & $n=10$ & 1.3647$\pm$0.0150 & 0.7839$\pm$0.0046 & 0.5633$\pm$0.0065  \\ 
		$\mathsf{FT}$ & $n=10$ & 1.6290$\pm$0.0214 & 0.8247$\pm$0.0048 & 0.6531$\pm$0.0085  \\
		$\mathsf{FT|Ref}$ & $n=10$ & \textbf{1.6439}$\pm$0.0249 & \textbf{0.8276}$\pm$0.0056 & \textbf{0.6605}$\pm$0.0095 \\ \midrule
		
		$\mathsf{TR\!-\!Ref}$ & EUB & 1.3822$\pm$0.0413 & 0.7910$\pm$0.0159 & 0.5708$\pm$0.0229 \\
		$\mathsf{FT}$ & EUB & 1.8695$\pm$0.0268 & 0.8488$\pm$0.0051 & 0.7389$\pm$0.0047 \\
		$\mathsf{FT|Ref}$ & EUB & \textbf{1.8831}$\pm$0.0189 & \textbf{0.8494}$\pm$0.0046 & \textbf{0.7442}$\pm$0.0058 \\
		\bottomrule	
	\end{tabular}
	\label{tbl:s_d_w_sup}
\end{table}

\begin{table}[!t]
	\centering
	\caption{Performance within setting $\langle$SALICON, ResNet-50, WebSal$\rangle$}
	\footnotesize
	\begin{tabular}{lcccc}
		\toprule
		& & NSS$\uparrow$ & AUC$\uparrow$ & CC$\uparrow$ \\
		\cmidrule(lr){3-3} \cmidrule(lr){4-4} \cmidrule(lr){5-5}
		$\mathsf{TR}$ & $n=0$ & ~~1.2950$\pm$0.0604~~ & ~~0.7749$\pm$0.0120~~ & ~~0.5358$\pm$0.0226~~  \\ \midrule
		
		$\mathsf{TR\!-\!Ref}$~~~~~~~ & $n=1$ & 1.3569$\pm$0.0236 & 0.7864$\pm$0.0055 & 0.5611$\pm$0.0099 \\
		$\mathsf{FT}$ & $n=1$ & 1.3722$\pm$0.0611 & 0.7923$\pm$0.0169 & 0.5627$\pm$0.0254 \\
		$\mathsf{FT|Ref}$ & $n=1$ & \textbf{1.4272}$\pm$0.0565 & \textbf{0.7983}$\pm$0.0133 & \textbf{0.5817}$\pm$0.0224 \\ \midrule
		
		$\mathsf{TR\!-\!Ref}$ & $n=5$ & 1.3535$\pm$0.0159 & 0.7837$\pm$0.0080 & 0.5593$\pm$0.0078  \\
		$\mathsf{FT}$ & $n=5$ & 1.5043$\pm$0.0165 & 0.8131$\pm$0.0051 & 0.6139$\pm$0.0064  \\
		$\mathsf{FT|Ref}$ & $n=5$ & \textbf{1.5491}$\pm$0.0242 & \textbf{0.8149}$\pm$0.0083 & \textbf{0.6281}$\pm$0.0092 \\ \midrule
		
		$\mathsf{TR\!-\!Ref}$ & $n=10$ & 1.3583$\pm$0.0242 & 0.7857$\pm$0.0056 & 0.5612$\pm$0.0103  \\
		$\mathsf{FT}$ & $n=10$ & 1.5164$\pm$0.0224 & 0.8103$\pm$0.0051 & 0.6200$\pm$0.0078  \\
		$\mathsf{FT|Ref}$ & $n=10$ & \textbf{1.5829}$\pm$0.0282 & \textbf{0.8143}$\pm$0.0082 & \textbf{0.6414}$\pm$0.0100 \\ \midrule
		
		$\mathsf{TR\!-\!Ref}$ & EUB & 1.3626$\pm$0.0082 & 0.7864$\pm$0.0110 & 0.5645$\pm$0.0031 \\
		$\mathsf{FT}$ & EUB & 1.8325$\pm$0.0414 & 0.8462$\pm$0.0016 & 0.7275$\pm$0.0066 \\
		$\mathsf{FT|Ref}$ & EUB & \textbf{1.8500}$\pm$0.0391 & \textbf{0.8480}$\pm$0.0016 & \textbf{0.7321}$\pm$0.0049 \\
		\bottomrule	
	\end{tabular}
	\label{tbl:s_r_w_sup}
\end{table}

\begin{table}[!t]
	\centering
	\caption{Performance within setting $\langle$MIT1003, DINet, WebSal$\rangle$}
	\footnotesize
	\begin{tabular}{lcccc}
		\toprule
		
		& & NSS$\uparrow$ & AUC$\uparrow$ & CC$\uparrow$ \\
		\cmidrule(lr){3-3} \cmidrule(lr){4-4} \cmidrule(lr){5-5}
		$\mathsf{TR}$ & $n=0$ & ~~1.3905$\pm$0.0062~~ & ~~0.7991$\pm$0.0010~~ & ~~0.5700$\pm$0.0021~~ \\ \midrule
		
		$\mathsf{TR\!-\!Ref}$ & $n=1$ & 1.4405$\pm$0.0175 & 0.8085$\pm$0.0026 & 0.5902$\pm$0.0071 \\
		$\mathsf{FT}$ & $n=1$ & 1.4410$\pm$0.0941 & 0.8023$\pm$0.0148 & 0.5784$\pm$0.0397 \\
		$\mathsf{FT|Ref}$ & $n=1$ & \textbf{1.4575}$\pm$0.0884 & \textbf{0.8070}$\pm$0.0140 & \textbf{0.5838}$\pm$0.0369 \\ \midrule
		
		$\mathsf{TR\!-\!Ref}$~~~~~~~ & $n=5$ & 1.4452$\pm$0.0156 & 0.8064$\pm$0.0025 & 0.5908$\pm$0.0055  \\
		$\mathsf{FT}$ & $n=5$ & 1.5795$\pm$0.0322 & 0.8217$\pm$0.0058 & 0.6395$\pm$0.0134  \\
		$\mathsf{FT|Ref}$ & $n=5$ & \textbf{1.6136}$\pm$0.0304 & \textbf{0.8269}$\pm$0.0045 & \textbf{0.6515}$\pm$0.0116 \\ \midrule
		
		$\mathsf{TR\!-\!Ref}$ & $n=10$ & 1.4330$\pm$0.0165 & 0.8060$\pm$0.0033 & 0.5872$\pm$0.0055  \\
		$\mathsf{FT}$ & $n=10$ & 1.6462$\pm$0.0216 & 0.8261$\pm$0.0044 & 0.6660$\pm$0.0083  \\
		$\mathsf{FT|Ref}$ & $n=10$ & \textbf{1.6691}$\pm$0.0206 & \textbf{0.8283}$\pm$0.0039 & \textbf{0.6730}$\pm$0.0068 \\ \midrule
		
		$\mathsf{TR\!-\!Ref}$ & EUB & 1.4402$\pm$0.0410 & 0.8087$\pm$0.0035 & 0.5905$\pm$0.0108 \\
		$\mathsf{FT}$ & EUB & 1.8450$\pm$0.0430 & 0.8466$\pm$0.0036 & 0.7330$\pm$0.0086 \\
		$\mathsf{FT|Ref}$ & EUB & \textbf{1.8507}$\pm$0.0397 & \textbf{0.8478}$\pm$0.0038 & \textbf{0.7344}$\pm$0.0080 \\
		\bottomrule	
	\end{tabular}
	\label{tbl:m_d_w_sup}
\end{table}

\begin{table}[!t]
	\centering
	\caption{Performance within setting $\langle$SALICON, DINet, Art$\rangle$}
	\footnotesize
	\begin{tabular}{lcccc}
		\toprule
		
		& & NSS$\uparrow$ & AUC$\uparrow$ & CC$\uparrow$ \\
		\cmidrule(lr){3-3} \cmidrule(lr){4-4} \cmidrule(lr){5-5} 
		$\mathsf{TR}$ & $n=0$ & ~~1.5172$\pm$0.0525~~ & ~~0.8225$\pm$0.0063~~ & ~~0.6003$\pm$0.0184~~ \\ \midrule
		
		$\mathsf{TR\!-\!Ref}$~~~~~~~ & $n=1$ & 1.5651$\pm$0.0299 & 0.8287$\pm$0.0045 & 0.6211$\pm$0.0105 \\
		$\mathsf{FT}$ & $n=1$ & 1.6255$\pm$0.0365 & 0.8324$\pm$0.0095 & 0.6449$\pm$0.0153 \\
		$\mathsf{FT|Ref}$ & $n=1$ & \textbf{1.6523}$\pm$0.0397 & \textbf{0.8380}$\pm$0.0088 & \textbf{0.6564}$\pm$0.0150 \\ \midrule
		
		$\mathsf{TR\!-\!Ref}$ & $n=5$ & 1.5870$\pm$0.0238 & 0.8304$\pm$0.0029 & 0.6274$\pm$0.0093  \\
		$\mathsf{FT}$ & $n=5$ & 1.8049$\pm$0.0317 & 0.8480$\pm$0.0056 & 0.7185$\pm$0.0161  \\
		$\mathsf{FT|Ref}$ & $n=5$ & \textbf{1.8314}$\pm$0.0341 & \textbf{0.8503}$\pm$0.0051 & \textbf{0.7274}$\pm$0.0152 \\ \midrule
		
		$\mathsf{TR\!-\!Ref}$ & $n=10$ & 1.5704$\pm$0.0377 & 0.8288$\pm$0.0036 & 0.6204$\pm$0.0139  \\
		$\mathsf{FT}$ & $n=10$ & 1.8325$\pm$0.0819 & 0.8474$\pm$0.0050 & 0.7288$\pm$0.0301  \\
		$\mathsf{FT|Ref}$ & $n=10$ & \textbf{1.8584}$\pm$0.0949 & \textbf{0.8503}$\pm$0.0043 & \textbf{0.7366}$\pm$0.0341 \\ \midrule
		
		$\mathsf{TR\!-\!Ref}$ & EUB & 1.5980$\pm$0.0347 & 0.8340$\pm$0.0043 & 0.6331$\pm$0.0152 \\
		$\mathsf{FT}$ & EUB & 2.1595$\pm$0.0764 & 0.8636$\pm$0.0063 & 0.8464$\pm$0.0054 \\
		$\mathsf{FT|Ref}$ & EUB & \textbf{2.1874}$\pm$0.0795 & \textbf{0.8649}$\pm$0.0065 & \textbf{0.8519}$\pm$0.0110 \\
		\bottomrule	
	\end{tabular}
	\label{tbl:s_d_a_sup}
\end{table}

\newpage

\bibliographystyle{splncs04}
\bibliography{reference}